\documentclass[letterpaper, 10 pt, conference]{ieeeconf}  

\IEEEoverridecommandlockouts                              
\usepackage{graphics} 
\usepackage{mathptmx} 
\usepackage{times} 
\usepackage{amsmath} 
\usepackage{amssymb}  

\usepackage{acronym} 
\usepackage{algorithm} 
\usepackage{algpseudocode} 
\usepackage{cite} 
\usepackage{hyperref} 
\usepackage[short]{optidef} 
\usepackage{xcolor} 
\usepackage{subcaption} 
\usepackage{booktabs} 
\usepackage{footnote}
\usepackage{balance} 
\makesavenoteenv{tabular}
\makesavenoteenv{table}
\usepackage{tikz} 

\makeatletter
\def\endthebibliography{%
	\def\@noitemerr{\@latex@warning{Empty `thebibliography' environment}}%
	\endlist
}
\makeatother

\graphicspath{{figures/}}
\newcommand{\executeiffilenewer}[3]{%
	\ifnum\pdfstrcmp{\pdffilemoddate{#1}}%
	{\pdffilemoddate{#2}}>0%
	{\immediate\write18{#3}}\fi%
}

\newcommand{\priority}[2][1]{%
	\ifnum#1<10
		\ifnum#1<1
		\else%
		#2%
		\fi%
	\else%
	#2%
	\fi%
}

\renewcommand{\vec}[1]{\mathbf{ #1 }}

\newcommand{\set}[1]{\mathcal{ #1 }}
\newcommand*{\bmatr}[1]{\begin{bma`x} #1 \end{bmatrix}}
\newcommand*{\matr}[1]{\begin{matrix} #1 \end{matrix}}
\newcommand{\ub}[1]{\overline{ #1 }}
\newcommand{\lb}[1]{\underline{ #1 }}
\newcommand{\R}{\mathbb{R}}
\newcommand{\T}{^\top}
\newcommand{\norm}[1]{\left\lVert #1 \right\rVert}
\newcommand{\state}{\vec{x}}
\newcommand{\controls}{\vec{u}}
\newcommand{\inistate}{\overline{\underline{\state}}_0}
\newcommand{\nstate}{{n_{\mathrm{\state}}}}
\newcommand{\ncontrols}{{n_{\mathrm{\controls}}}}
\newcommand{\refstate}{\hat{\state}}
\newcommand{\refcontrols}{\hat{\controls}}
\newcommand{\occupied}{\set{O}}
\newcommand{\dist}{d_\occupied}
\newcommand{\dmin}{\lb{d}}
\newcommand{\dmax}{\ub{d}}

\newcommand{\workspace}{\set{W}}
\newcommand{\admissible}{\set{A}}
\newcommand{\pos}{\vec{p}}
\newcommand{\cen}{\vec{c}}

\newtheorem{definition}{Definition}
\newtheorem{lemma}{Lemma}

\algdef{SE}[DOWHILE]{Do}{DoWhile}{\algorithmicdo}[1]{\algorithmicwhile\ #1}%

\title{\LARGE \bf
An NMPC Approach using Convex Inner Approximations for\\Online Motion Planning with Guaranteed Collision Avoidance
}

\author{Tobias Schoels\textsuperscript{1,2}, Luigi Palmieri\textsuperscript{2}, Kai O. Arras\textsuperscript{2}, and Moritz Diehl\textsuperscript{1}
	\thanks{\textsuperscript{1}T.~Schoels and M.~Diehl are with the Department of Microsystems Engineering, University of Freiburg.
		\texttt{\{tobias.schoels, moritz.diehl\}@imtek.uni-freiburg.de}.}
	\thanks{\textsuperscript{2}T.~Schoels, L.~Palmieri and K.~O.~Arras are with Robert Bosch GmbH, Corporate Research, Stuttgart, Germany.
		\texttt{\{tobias.schoels, luigi.palmieri, kaioliver.arras\}@de.bosch.com}.}
	\thanks{This research was supported by the German Federal Ministry for Economic Affairs and Energy (BMWi) via eco4wind (0324125B) and DyConPV (0324166B), by DFG via Research Unit FOR 2401, and the EU’s Horizon
		2020 research and innovation program under grant agreement No 732737
		(ILIAD).}
}

\begin{document}

	\maketitle
	\thispagestyle{empty}
	\pagestyle{empty}

	\begin{abstract}
		Even though mobile robots have been around for
		decades, trajectory optimization and continuous time collision
		avoidance remain subject of active research. Existing methods
		trade off between path quality, computational complexity, and
		kinodynamic feasibility. This work approaches the problem using
		a nonlinear model predictive control (NMPC) framework, that is based
		on a novel convex inner approximation of the collision avoidance
		constraint. The proposed Convex Inner ApprOximation (CIAO)
		method finds kinodynamically feasible and continuous time collision
		free trajectories, in few iterations, typically one.
		For a feasible initialization, the approach is guaranteed to
		find a feasible solution, i.e. it preserves feasibility. Our
		experimental evaluation shows that CIAO outperforms
		state of the art baselines in terms of planning efficiency and
		path quality. Experiments on a robot with 12 states show
		that it also scales to high-dimensional systems. Furthermore
		real-world experiments demonstrate its capability of unifying
		trajectory optimization and tracking for safe motion planning
		in dynamic environments.
	\end{abstract}

	\section{INTRODUCTION}
	Several existing mobile robotics applications (e.g. intra-logistic and service robotics) require robots to operate in dynamic environments among other agents, such as humans or other autonomous systems. In these scenarios, the reactive avoidance of unforeseen dynamic obstacles is an important requirement.
	Combined with the objective of reaching optimal robot behavior, this poses
	a major challenge for motion planning and control and remains subject of
	active research.

	Recently several researchers have tackled the obstacle avoidance problem by formulating and solving optimization problems \cite{Quinlan1993, Brock2002, Zucker2013, Schulman2014, Herbert2017, Zhang2017, Bonalli2019, Verscheure2009, Zhu2015, Roesmann2017,  Frasch2013b, Liniger2015, Faulwasser2016, Neunert2016}.
	This approach is well suited for finding locally optimal solutions, but generally gives no guarantee of finding the global optimum.
	Most methods therefore rely on the initialization by an asymptotically optimal sampling-based planner \cite{Karaman2011,palmieri2014novel,palmieri2016rrt}.
	A shortcoming of most common trajectory optimization methods is that they
	are incapable of respecting kinodynamic constraints, e.g. bounds on the
	acceleration, and typically lack a notion of time in their predictions,
	\cite{Quinlan1993, Brock2002, Zucker2013, Schulman2014}.
	These approaches are typically limited to the optimization of paths rather than trajectories and impose constraints by introducing penalties.

	The increase of computing power and the availability of fast numerical solvers, as discussed in \cite{Kouzoupis2018}, has given rise to \ac{mpc} based approaches, e.g. \cite{ 
		Zucker2013, Schulman2014, Zhang2017, Herbert2017, Bonalli2019
	}.
	In this framework, an \ac{ocp} is solved in every iteration. These methods
	succeed in finding kinodynamically \cite{Zhang2017, Bonalli2019,
	Herbert2017} or kinematically \cite{Zucker2013, Schulman2014, Herbert2017}
	feasible trajectories, but typically use penalty terms in the cost function that offer no safety guarantees \cite{Bonalli2019} or
	require that obstacles are given as a set of convex hulls
	\cite{Schulman2014}. 
	\begin{figure}
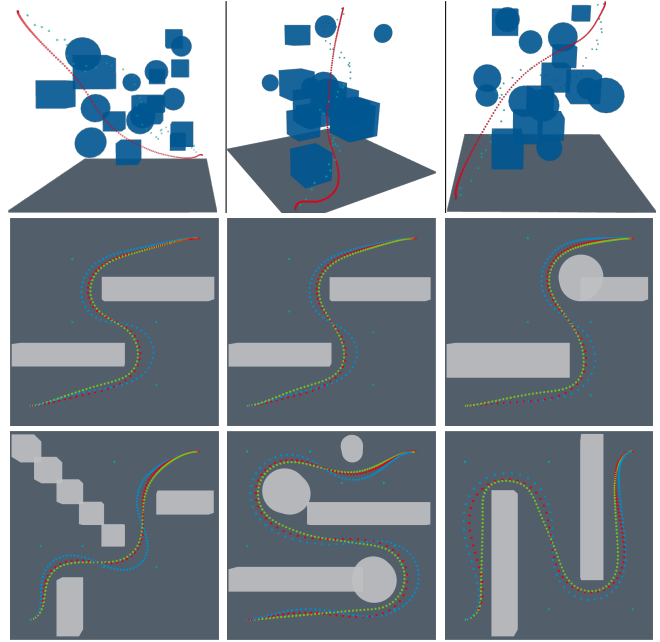

		\centering
		\includegraphics[width=0.32\linewidth]{figures/3D-example1.png} \vline \hfil
		\includegraphics[trim={0mm 3mm 0mm 0mm}, clip, width=0.32\linewidth]{figures/3D-example2a.png} \vline \hfil
		\includegraphics[width=0.32\linewidth]{figures/3D-example4.png} \\[2pt]
		\includegraphics[width=0.32\linewidth]{figures/2D-example4.png}
		\includegraphics[width=0.32\linewidth]{figures/2D-example5.png}
		\includegraphics[width=0.32\linewidth]{figures/2D-example6.png} \\[2pt]
		\includegraphics[width=0.32\linewidth]{figures/2D-example7.png}
		\includegraphics[width=0.32\linewidth]{figures/2D-example8.png}
		\includegraphics[width=0.32\linewidth]{figures/2D-example9.png}

		\caption{\ac{sciam} trajectories for the Astrobee robot (top row) in red and for a unicycle robot (last two rows) with three different maximum speeds $v_{\max}$ and corresponding minimum distances $\dmin$ (see \eqref{eq:ctime_collision_freedom}): green - slow, red - normal, blue - fast. The boxes and spheres represent obstacles, the turquoise dots the reference path. A wider spacing between the dots indicates a higher speed. The start is always located in the bottom and the goal in the top.
			It is clearly visible that \ac{sciam} maintains higher distances to obstacles for higher speeds.}
		\label{fig:benchmarkarena}
	\end{figure}%
	%

	\paragraph*{Contribution}
  	This work presents \ac{sciam}, a \ac{nmpc} based approach to real-time collision avoidance for single body robots.
	It preserves feasibility across iterations and uses a novel, convex formulation of the collision avoidance constraint
	that is compatible with many implementations of the distance function, even discrete ones like distance fields.
	To the best of authors' knowledge, CIAO is the first real-time capable \ac{nmpc} approach that
	guarantees continuous time collision free trajectories and is agnostic of the distance function's implementation.
	The method's efficacy is demonstrated and evaluated in simulation and real-world using robots with nonlinear, constrained dynamics and state of the art baselines.

	\paragraph*{Structure}
	The paper is structured as follows: The related work is discussed in Section \ref{sec:relatedwork} and Section \ref{sec:problem} introduces the problem we want to solve.
	Section \ref{sec:sciam} details
	\ac{sciam} alongside some considerations on feasibility, safety and
	practical challenges. In Section \ref{sec:motion_planning} we detail how
	\ac{sciam} can be used for trajectory optimization and \ac{rhc}. The
	experiments and results are discussed in Section \ref{sec:experiments}. A
	summary and an outlook is given in Section \ref{sec:conclusion}.
	\section{Related Work}
	\label{sec:relatedwork}
	%
	Trajectory optimization methods try to find time-optimal and collision-free robot trajectories by formulating and solving an optimization problem
	\cite{
		Quinlan1993, Brock2002, Verscheure2009, Zucker2013, Schulman2014, Herbert2017, Zhang2017, Roesmann2017, Bonalli2019, Zhu2015, Frasch2013b, Liniger2015, Faulwasser2016, Neunert2016
	}. 
	Classical approaches to obstacle avoidance include \cite{Borenstein1991, Fox1997, Ko1998, Fiorini1998, Minguez2004}. 
	These approaches do neither produce optimal trajectories, nor unify planning and control, nor account for complex robot dynamics.

	A simple and effective method that is still used in practice, is the elastic-band algorithm \cite{Quinlan1993}.
	The computed paths, however, are generally non-smooth, i.e. they are not guaranteed to satisfy kinodynamic constraints, nor does this algorithm compute a velocity profile.
	Like Zhu et al. \cite{Zhu2015}, our approach aims to fix this shortcoming,
	while building up on the notion of (circular) free regions. Instead of
	optimizing velocity profile and path length separately, as done in
	\cite{Zhu2015,Quinlan1993}, we optimize them jointly utilizing an \acf{nmpc}
  setup. Also R{\"o}smann et al. \cite{Roesmann2017} provide an approach which
	combines elastic-band with an optimization algorithm. Contrarily to ours,
	their approach does not enforce obstacle avoidance as constraint, and
	requires a further controller for trajectory tracking.

	Nowadays optimization based methods receive increasingly more attention.
	Popular methods include CHOMP \cite{Zucker2013}, TrajOpt \cite{Schulman2014}, OBCA \cite{Zhang2017}, and GuSTO \cite{Bonalli2019}, which have been shown to produce smooth trajectories efficiently. They typically use a simplified system model to compute a path from the current to the goal state (or set), and further need an additional controller to steer the system along the precomputed path.
	The proposed method, \ac{sciam}, is \ac{mpc}-based and provides algorithms for both trajectory optimization and \ac{rhc}, simultaneously controlling the robot and optimizing its trajectory.

	Also Neunert et al. \cite{Neunert2016} propose \ac{rhc} to unify trajectory optimization and tracking, but their approach does not include a strategy for obstacle avoidance. They propose to solve an unconstrained \ac{nlp} online, \ac{sciam} on the other hand is solving a constrained \ac{nlp} online that enforces collision avoidance as constraint.

	Frasch et al. \cite{Frasch2013b} propose an \ac{mpc} with box-constraints to model obstacles and road boundaries. Liniger et al. \cite{Liniger2015} handle obstacles in a similar way, but apply contouring control, i.e. the approach steers a race car in a corridor around a predefined path.
	These frameworks do not consider arbitrarily placed obstacles, particularly
	no moving obstacles, which makes them unsuitable for many applications. In
	our approach we handle more complex obstacle definitions, modeled according
	to a generic nonlinear and nonconvex distance function. To this end, we
	propose a novel constraint formulation that is shown to be a convex inner
	approximation of the actual collision avoidance constraint.

  Herbert et al. \cite{Herbert2017} propose a hybrid approach to safely avoid dynamic
	obstacles. The trajectory tracker does not consider the
	obstacles explicitly, but relies on the planning layer.
  The method proposed in this work presents a unified approach for all robot motion planing,
	control, and obstacle avoidance, using constrained \ac{nmpc}.
	The recently proposed method GuSTO \cite{Bonalli2019} uses \ac{scp}, like other common algorithms, e.g. \cite{Schulman2014, Liniger2015, Frasch2013b, Roesmann2017}.
	\ac{scp} requires a full convexification of the originally nonlinear and nonconvex
	trajectory optimization problem.
	This is accomplished by linearizing the
	system model and incorporating paths constraints,
	including collision avoidance, as penalties in the objective function. In
	general these approximations may lead to infeasible, i.e. colliding or kinodynamically
	intractable trajectories. Typically several \ac{scp} iterations are required
	to find a feasible solution.
	\ac{sciam}, on the other hand, solves partially convexified \acp{nlp}, using a convex inner approximation of the
	collision avoidance constraint, and finds feasible solutions in less iterations, typically one.
	The individual iterations are computationally cheaper and feasibility is preserved.
	Since the dynamical model is accounted for by the \ac{nlp}-solver, linearization errors are minimized.
	%
%
	%

	Finally \ac{sciam} can be considered as a trust region method \cite{Yuan2015}, where the nonlinear state constraints are approximated with a convex inner approximation.
	\section{PROBLEM FORMULATION}\label{sec:problem}

	We want to find a kinodynamically feasible, collision free trajectory by formulating and solving a constrained \ac{ocp}.
	Kinodynamic feasibility is ensured by using a dynamical model to simulate the robot's behavior and collision avoidance is achieved constraining the robot to positions with a \emph{minimum distance} $\dmin$ to all obstacles.

	The \emph{occupied set} $\occupied$ is defined as the set of points in the robot's workspace $\workspace\subseteq\R^n$ that are occupied by obstacles.
	The Euclidean distance to the closest obstacle for any point $\pos \in \workspace$ is given by the \emph{distance
	function} $\dist: \workspace
	\to \R$:
	\begin{equation}\label{eq:dist}
		\dist(\pos) = d(\pos; \occupied)= \underset{\mathbf{o} \in \occupied}{\min}  \norm{\pos - \mathbf{o} }_2.
	\end{equation}
	For the sake of simplicity we assume that the robot's shape is contained in an $n$-dimensional sphere with radius $< \dmin$ and center point $\pos$.
	Collision avoidance can now be achieved by requesting that
	the distance function at the robot's position $\pos \in \workspace$ is at least the minimum distance $\dmin$:
	\begin{equation}\label{eq:min_dist}
		\dist(\pos) \geq \dmin,
	\end{equation}
	which is equivalent to $\| \pos - \mathbf{o} \|_2 \geq \dmin, \forall\, \mathbf{o} \in \occupied$.

	We assume $\dmin > 0$ to be fixed from now on, to ensure that points in the occupied set $\occupied$ do not satisfy \eqref{eq:min_dist}. A point $\pos \in \workspace$ that satisfies \eqref{eq:min_dist} is called `\emph{free}'.
	\priority{Further we define the \emph{safety margin} as the area for which $0 < \dist(\pos) <  \dmin $ holds.}

	We can now formulate an \ac{ocp} that enforces collision avoidance as path constraint:
	\begin{mini}
		{\state(\cdot), \controls(\cdot)}
		{\int_0^T l(\state(t), \controls(t), \mathbf{r}(t))\; dt \; + l_\mathrm{T}(\state(T), \mathbf{r}(T))} 
		{\label{eq:ocp}} 
		{} 
		\addConstraint{\state(0)}{= \inistate}{}
		\addConstraint{\state(T)}{\in \mathbb{X}_\mathrm{T}}{}
		\addConstraint{\dot{\state}(t)}{= f(\state(t), \controls(t)),}{\quad t \in [0, T]}
		\addConstraint{h(\state(t), \controls(t))}{\leq 0,}{\quad t \in [0, T]}
		\addConstraint{\dist(\pos(t))}{\geq \dmin,}{\quad t \in [0, T],}
	\end{mini}
	where $\state(\cdot): \R \rightarrow \R^{\nstate}$ denotes the robot's state, $\controls(\cdot): \R \rightarrow \R^{\ncontrols}$ is the vector of controls, $\mathbf{r}(\cdot):\R\to\R^{\nstate+\ncontrols}$ provides reference states and controls, $T$ is the length of the horizon in seconds, the function $l(\cdot)$ denotes the cost at time point $t$ and $l_\mathrm{T}(\cdot)$ the terminal cost, $\inistate$ is robot's current state, and $\mathbb{X}_\mathrm{T} \subseteq \R^{\nstate}$ is the set of admissible terminal states.
	We use the common shorthand $\dot{\state}$ to denote the derivative with respect to time, i.e. $\dot{\state} = \frac{\partial\, \state}{\partial t}$.
  The function $f$ models the robot's dynamics and $h$ implements a set of path constraints, e.g. physical limitations of the system, and $\pos(t) = S_\pos\cdot\state(t) \in \workspace$ denotes the robot's position
	with a selector matrix $S_\pos$ chosen accordingly.

	\section{CONVEX INNER APPROXIMATION (CIAO)} \label{sec:sciam}
	In this section we describe how we solve the \ac{ocp} presented in \eqref{eq:ocp} by adopting a convex inner approximation of the actual collision avoidance constraint presented in Sec.~\ref{sec:problem}.

	First we discretize \eqref{eq:ocp} using a direct multiple shooting scheme as proposed by \cite{Bock1984}. 
	The resulting \ac{nlp} is a function of $\vec{r}$, $\inistate$, and the
	sampling time $\Delta t$. Using the shorthand $\state_k = \state(k \cdot
	\Delta t),~k \in \mathbb{Z}$ for cleaner notation, we discretize \eqref{eq:ocp}
	as:\vspace{-1em}
	\begin{mini!}
		{\vec{w}}{J(\vec{w}, \vec{r})}
		{\label{eq:nlp_init}}{}
		\addConstraint{\state_0 -  \inistate}{=0}
		\addConstraint{\state_N}{\in \mathbb{X}_\mathrm{T}}{}
		\addConstraint{\state_{k+1} - F(\state_k, \controls_k; \Delta t)}{=0,}{\quad k=0,\ldots,N-1}
		\addConstraint{h(\state_k, \controls_k)}{\leq 0,}{\quad k=0,\ldots,N}
		\addConstraint{\dist(\pos_k)}{\geq\dmin, \label{eq:nlp_cac}}{\quad k=0,\ldots,N,}
	\end{mini!}
	where $\vec{w} = [\state_0\T, \controls_0\T, \ldots, \controls_{N-1}\T, \state_N\T]\T \in \R^{n_\vec{w}}$ is a vector of optimization variables that contains the stacked controls and states for all $N$ steps in the horizon, similarly $\vec{r}$ contains the reference states and controls, $J(\vec{w}, \vec{r}) = \sum_{k=0}^{N-1} l(\state_k, \controls_k, \vec{r}_k) + l_\mathrm{T}(\state_N, \vec{r}_N)$ is the discretized objective, with stage cost $l$ and terminal cost $l_\mathrm{T}$, and $F$ models the discretized robot dynamics. 
	We denote the feasible set of \eqref{eq:nlp_init} by $\set{F}_\eqref{eq:nlp_init} \subset \R^{n_\vec{w}}$.

	\subsection{Free Balls: A Convex Inner Approximation of the Obstacle Avoidance Constraint }
	\label{sec:cfr}
	\begin{figure}
		\centering
		\vspace{0.2cm}
		\begin{subfigure}[t]{0.45\columnwidth}
			\centering
			\def\svgwidth{0.7\columnwidth}
	\executeiffilenewer{figures/cfr.svg}{figures/cfr.pdf}%
	{inkscape -z -D --file=/Users/tobi/Documents/pubs/schoels-iros-2019/figures/cfr.svg %
		--export-pdf=/Users/tobi/Documents/pubs/schoels-iros-2019/figures/cfr.pdf --export-latex}%
	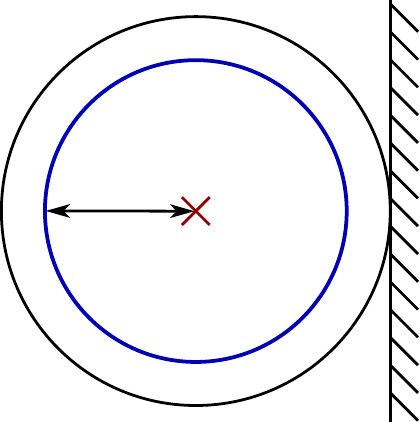%

			\caption[\ac{cfr}]{Example of a \acl{cfr} around the center $\cen$ marked by the red cross for a 2D environment. It is also the center of the circles, the black circle has radius $\dist(\cen)$, and the blue circle radius $\dist(\cen) - \dmin$.}
			\label{fig:cfr}
		\end{subfigure}
		\hspace{1em}
		\begin{subfigure}[t]{0.45\columnwidth}
			\centering
			\def\svgwidth{0.7\columnwidth}
	\executeiffilenewer{figures/cfr_constraint.svg}{figures/cfr_constraint.pdf}%
	{inkscape -z -D --file=/Users/tobi/Documents/pubs/schoels-iros-2019/figures/cfr_constraint.svg %
		--export-pdf=/Users/tobi/Documents/pubs/schoels-iros-2019/figures/cfr_constraint.pdf --export-latex}%
	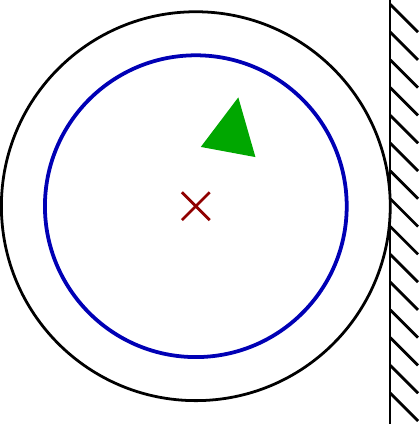%

			\caption[\ac{cfr} Constraint]{Example for \acl{cfr} constraint. The green arrow head depicts the robot's current position $\pos$, the orange dot the closest obstacle $\mathbf{o}$, the circles, and the red cross are identical to the ones in \ref{fig:cfr}.} 
			\label{fig:cfr_constraint}
		\end{subfigure}
		\caption{The left figure illustrates the \acl{cfr} concept, the right shows how it can be used as a constraint.}
		\label{fig:cfr_both}
	\end{figure}
	The actual obstacle avoidance constraint formulated in Eq.~\eqref{eq:min_dist} is generally nonconvex and nonlinear, which makes it ill-suited for rapid optimization. We propose a convex inner approximation of the constraint that is based on the notion of \acfp{cfr}, as proposed in \cite{Quinlan1993} and extended by \cite{Zhu2015}.
	For cleaner notation we first define the \emph{free set}.

	\begin{definition}\label{def:admissible_set}
		Let $\occupied$ be the occupied set and $\dmin > 0$ be the minimum distance,
		then the free set $\admissible$ is defined as\par
		$
		\hfil\admissible = \{a \in \workspace : \| \mathbf{a} - \mathbf{o} \|_2 \geq \dmin \quad \forall \; \mathbf{o} \in \occupied \}.
		$
	\end{definition}
	\textit{Remark:}
	This definition implies that the free set $\admissible$ and the occupied set $\occupied$ are disjunct, i.e. $\admissible \cap \occupied = \emptyset$.

	We can now formulate an obstacle avoidance constraint by enforcing that the robot's position lies within an $n$-dimensional ball formed around $\cen \in \admissible$ as shown in Fig.~\ref{fig:cfr_both}.
	\begin{definition}\label{def:cfr}
		For an arbitrary free point $\cen \in \admissible$ we define the \textit{\acl{cfr}} as 
		$ \admissible_{\cen} := \left\{ \pos \in \workspace : \| \pos - \cen \|_2 \leq \dist(\cen) - \dmin \right\}.$
	\end{definition}

	We will now show that a \acl{cfr} is a convex subset of the free set.

	\begin{lemma} \label{lem:inner_approximation}
		Let $\cen \in \admissible$ be a free point, then the \acl{cfr} $\admissible_{\cen}$ is a convex subset of $\admissible$, i.e. $\cen \in \admissible \Rightarrow \admissible_{\cen} \subseteq \admissible$.

	\end{lemma}
	\begin{proof}
		We will prove this lemma in two steps. First, we observe that the \acl{cfr} is a norm ball and therefore convex. Second, we show by contradiction that $\admissible_{\cen} \nsubseteq \admissible \Rightarrow \cen \notin \admissible$.

		Suppose $\exists\; \mathbf{o} \in \occupied$ and $\pos \in \admissible_{\cen}$ such that $\norm{\pos - \mathbf{o}}_2 < \dmin$. 
		We now apply the triangle inequality and obtain
		$\norm{\cen - \mathbf{o}}_2 \leq \norm{\pos - \cen}_2 + \norm{\pos - \mathbf{o}}_2 < \norm{\pos - \cen}_2 + \dmin$, see Fig.~\ref{fig:cfr_constraint}. 
		Using the distance function's definition we get $\dist(\cen) < \norm{\pos - \cen}_2 + \dmin$. Reordering yields $\norm{\pos - \cen}_2 > \dist(\cen) - \dmin$, which shows that $\admissible_{\cen} \nsubseteq \admissible$.
	\end{proof}


		Based on Lem.~\ref{lem:inner_approximation} and Def.~\ref{def:cfr} we can approximate the collision avoidance constraint by
		$
		\norm{\pos -\cen}_2 \leq \dist(\cen) - \dmin.
		$ 
		This formulation is not differentiable in $\pos = \cen$ and might pose a problem for gradient based solvers.
		To prevent this case and \ac{licq} violations at the only feasible point we assume
		$
		\dist(\cen) > \dmin.
		$
		This implies that both sides are grater $0$, such that we can square both sides and get
		\begin{equation}\label{eq:cfr_constraint}
		\norm{\pos -\cen}^2_2 \leq (\dist(\cen) - \dmin)^2.
		\end{equation}

	With the constraint formulated in \eqref{eq:cfr_constraint} and assuming
	that \ac{licq} holds for all \acl{cfr} center points
	$\cen_k$ with $k = 0, \ldots , N$, we can partially convexify the \ac{nlp}
	\eqref{eq:nlp_init}. We obtain the \ac{sciam}-\ac{nlp}, which like
	\eqref{eq:nlp_init} depends on $\vec{r}$, $\inistate$, $\Delta t$, and
	additionally the tuple of center points $\mathbf{C} = ( \cen_0, \ldots,
	\cen_N )$:
	%
	%
	\begin{mini!}[0]
		{\vec{w}, \vec{s}}{J(\vec{w}, \vec{r}) + \sum_{k=0}^N \mu_k \cdot s_k}
		{\label{eq:nlp_cfr}}{}
		\addConstraint{\state_0- \inistate=0}{\hspace{0pt}}
		\addConstraint{\state_N \in \mathbb{X}_\mathrm{T}}{\hspace{0pt} \label{eq:nlp_cfr_terminal_constraint}}
		\addConstraint{\state_{k+1}- F(\state_k, \controls_k; \Delta t)=0}{, \quad}{ k=0,\ldots,N-1}
		\addConstraint{h(\state_k, \controls_k) \leq 0}{, \quad}{k=0,\ldots,N}
		\addConstraint{\hspace{-2.5em}\norm{\pos_k -\cen_k}^2_2 \leq (\dist(\cen_k) - \dmin_k)^2 + s_k}{, \;}{k=0,\ldots,N \label{eq:nlp_cfr_constraint}}
		\addConstraint{s_k \geq 0}{, \quad}{k=0,\ldots,N.}
	\end{mini!}
	This reformulation of the actual \ac{nlp} \eqref{eq:nlp_init} is called \acf{sciam}.
	For numerical stability we include slack variables $\vec{s} = [s_0, \ldots,
	s_N]\T \in \R^{N+1}$ that are penalized. A point $\vec{w}$ is only considered
	admissible if all slacks are zero, i.e. $\vec{s} = 0$\priority[2]{\ in a
	vector sense}. To ensure that the slacks are only active for problems, that
	would be infeasible otherwise, the multipliers $\mu_k$ have to be chosen
	sufficiently large, i.e. $\mu_k \gg 1$ for $k=0,\ldots, N$.
	The feasible set for optimization variables $\vec{w}$ of this \ac{nlp} depends on $\mathbf{C}$ and is denoted as $\set{F}_\eqref{eq:nlp_cfr}(\mathbf{C})$\priority[2]{, recall that $\vec{w} = [\state_0\T, \controls_0\T, \ldots, \controls_{N-1}\T, \state_N\T]\T$ and $\pos_k = S_\pos \cdot \state_k$}.

	Note that $\dist(\cen)$ enters the \ac{nlp} as a constant ($\cen$ is a parameter not an optimization variable). Thereby \ac{sciam} is compatible any implementation of the distance function, even discrete ones.

		Note that for a convex objective $J(\cdot)$, a convex terminal set $\mathbb{X}_\mathrm{T}$, affine dynamics $F$, and convex path constraints $h$, the \ac{sciam}-\ac{nlp} \eqref{eq:nlp_cfr} is convex.
		Further note that for a linear-quadratic objective $J(\cdot)$, affine-quadratic path constraints $h$, affine dynamics $F$, and a terminal set $\mathbb{X}_\mathrm{T}$ that can be written as either (i) an affine equality or (ii) an affine-quadratic inequality constraint, \ac{sciam}-\ac{nlp} \eqref{eq:nlp_cfr} is a \ac{qcqp}. If is also convex, it is a convex \ac{qcqp}.
	\begin{lemma}\label{lem:sciam_subset}
		Given $\dist(\cen) > \dmin \; \forall \; \cen \in \mathbf{C} \Rightarrow \set{F}_{\eqref{eq:nlp_cfr}}(\mathbf{C}) \subseteq \set{F}_{\eqref{eq:nlp_init}}$, i.e. each feasible point of the \ac{sciam}-\ac{nlp} \eqref{eq:nlp_cfr} is a feasible point of the original \ac{nlp} \eqref{eq:nlp_init}.
	\end{lemma}
	\begin{proof}
		We observe that \eqref{eq:nlp_init} and \eqref{eq:nlp_cfr} are
		identical except for the collision avoidance constraint
		\eqref{eq:nlp_cac} and \eqref{eq:nlp_cfr_constraint}. As stated above
		$\vec{s}=0$ holds for feasible points, thus the slacks $\vec{s}$
		can be ignored. As shown in Lem.~\ref{lem:inner_approximation}
		\eqref{eq:nlp_cfr_constraint} is a convex inner approximation of
		\eqref{eq:nlp_cac}, therefore $\set{F}_{\eqref{eq:nlp_cfr}}(\mathbf{C})
		\subseteq \set{F}_{\eqref{eq:nlp_init}}$ follows by construction.
	\end{proof}

	\subsection{The \ac{sciam}-iteration}
	We will now introduce the \ac{sciam}-iteration, as detailed in Alg.~\ref{alg:solve_sciam}.
	It takes a two step approach that first formulates the \ac{sciam}-\ac{nlp} \eqref{eq:nlp_cfr} by finding a tuple of center points $\mathbf{C} = ( \cen_0, \ldots, \cen_N )$ before solving it.
	\begin{algorithm}
		\small
		\begin{algorithmic}[1]
			\Function{\ac{sciam}-iteration}{$\vec{w}\, ; \; \vec{r}, \; \inistate, \; \Delta t$}
			\State $\mathbf{C} \gets (\cen_k = S_\pos \cdot \state_k$ for $k =0,\ldots,N)$ \Comment recall $\state_k \in \vec{w}$
			\State $\mathbf{C}^* \gets (\cen^* = \textsc{maximize\acs{cfr}}(\cen)$ for all $\cen \in \mathbf{C}$) \Comment solve \eqref{eq:grow_cfr}
			\State $\vec{w}^* \gets$ \textsc{solve\ac{nlp}}($\vec{w} ; \; \mathbf{C}^*, \; \vec{r}, \; \inistate, \; \Delta t$) \Comment{solve \eqref{eq:nlp_cfr}}
			\EndFunction \ \Return $\vec{w}^*$ \Comment{return newly found trajectory}
		\end{algorithmic}
		\caption{the \ac{sciam}-iteration}
		\label{alg:solve_sciam}
	\end{algorithm}

	In Line 2 we find an initial tuple of center points $\mathbf{C}$. In
	practice the \aclp{cfr} resulting from these center points are very
	small and therefore very restrictive, which leaves little room for optimization, especially if the
	initial guess $\vec{w}$ approaches obstacles closely. To overcome this
	problem we maximize \acfp{cfr} (Line~3) by solving the following
	optimization problem for each $\cen \in \mathbf{C}$ and obtain an optimized
	center point $\cen^* = \eta \cdot \mathbf{g} + \cen$:
	\begin{equation}\label{eq:grow_cfr}
		\underset{\eta\geq 0}{\max} ~ \eta \quad \mathrm{s.t.} \quad \dist\left(\eta \cdot \mathbf{g} + \cen\right) = \eta + \dist(\cen),
	\end{equation}
	where $\cen \in \admissible$ is a given initial point, $\mathbf{g} \in \R^n$ is the search direction with $\norm{\mathbf{g}}_2 = 1$ and $\eta$ is the step size. It yields a maximized \acl{cfr} $\admissible_{\cen^*}$ with radius $r=\dist(\cen^*)$ and center point $\cen^* = \eta \cdot \mathbf{g} + \cen$ for each $\cen \in \mathbf{C}$. The optimized center points are collected in the tuple $\mathbf{C}^* = (\cen^*_0, \ldots \cen^*_N)$.
	\priority{To ensure convergence of Alg.\ref{alg:solve_sciam} we require that the distance function is bounded, i.e. $\exists\; \dmax > 0 $ such that $\dist(\pos)~\leq~\dmax \; \forall\; \pos \in \workspace$.}

	We will now show that the optimization problem \eqref{eq:grow_cfr} preserves feasibility of the initial guess $\vec{w}$ by showing that $\admissible_{\cen^*}$ includes $\admissible_{\cen}$, i.e. $\admissible_{\cen} \subseteq \admissible_{\cen^*}$.

	\begin{lemma} \label{lem:grow_cfr}
		For $\cen \in \admissible$, $\mathbf{g} \in \{g \in \R^n : \norm{g}_2 = 1 \} $ and $\eta \geq 0$,
		$\dist(\cen^*) = \eta + \dist(\cen) \Rightarrow \admissible_{\cen} \subseteq \admissible_{\cen^*}$ holds with $\cen^* = \eta \cdot \mathbf{g} + \cen$.
	\end{lemma}
	\begin{proof}
		We will prove this by contradiction, assuming $\exists \; \pos \in \admissible_{\cen}$ s.t. $\pos \notin \admissible_{\cen^*}$.
		Using Def.~\ref{def:cfr} we can rewrite this as $\norm{(\eta \cdot \mathbf{g} + \cen) - \pos}_2 > \dist(\cen^*) - \dmin$. Applying the triangle inequality on the left side yields $\norm{\pos - (\cen + \eta \cdot \bf{g})}_2 \leq \norm{\pos - \cen}_2 + \norm{\eta \cdot \bf{g}}_2 = \norm{\pos - \cen}_2 + \eta$ and based on our assumption $\norm{\pos - \cen}_2 + \eta \leq \dist(\cen) - \dmin + \eta$ holds. Inserting this gives $\dist(\cen) + \eta  - \dmin > \dist(\cen^*) - \dmin$ and thus contradicts the condition $\dist(\cen^*) = \eta + \dist(\cen)$.
	\end{proof}

	To solve the line search problem \eqref{eq:grow_cfr} we propose to use the distance function's normalized gradient $\mathbf{g} = \frac{\nabla \dist(\cen)}{\norm{\nabla \dist(\cen)}_2}$ as search direction.
	Starting from $\eta = \lb{\eta} > 0$ the step size is exponentially increased until a step size $\ub{\eta} > \lb{\eta}$ is found for which the constraint is violated.
  The optimal step size can now be found using the bisection method.
	\textsc{solveNLP} uses a suitable solver to solve \eqref{eq:nlp_cfr}, e.g. Ipopt \cite{Waechter2006}, and computes a new trajectory $\vec{w}^*$ (Line 4).
	\begin{lemma} \label{lem:recursive_feasibilty}
		For a feasible initial guess $\vec{w} \in  \set{F}_{\eqref{eq:nlp_init}}$ Alg.~\ref{alg:solve_sciam} finds a feasible point $\vec{w}^* \in \set{F}_{\eqref{eq:nlp_init}}$ with $J(\vec{w}^*) \leq J(\vec{w})$.
	\end{lemma}
	\begin{proof}
		We prove this in two steps: first we assume that \textsc{solve\ac{nlp}} uses a suitable, working \ac{nlp}-solver, then we show the feasibility.
		From $\admissible_{\cen} \subseteq \admissible_{\cen^*}$ as shown in Lem.~\ref{lem:grow_cfr} follows $\set{F}_{\eqref{eq:nlp_cfr}}(\mathbf{C}) \subseteq \set{F}_{\eqref{eq:nlp_cfr}}(\mathbf{C}^*)$. Further Lem.~\ref{lem:sciam_subset} yields $\set{F}_{\eqref{eq:nlp_cfr}}(\mathbf{C}) \subseteq \set{F}_{\eqref{eq:nlp_cfr}}(\mathbf{C}^*) \subseteq \set{F}_{\eqref{eq:nlp_init}}$.
	\end{proof}
	\subsection{Continuous Time Collision Avoidance for Systems with Bounded Acceleration} \label{sec:ctca}
	The constraints formulated in \eqref{eq:cfr_constraint} can be extended to the continuous time case. Using Lem.~\ref{lem:inner_approximation} we can write the continuous time collision avoidance constraint as
	\begin{equation*}
	\norm{\cen_k - \pos(t)}_2 \leq \dist(\cen_k) - \dmin, \quad \forall t \in [t_k, t_{k+1}], k = 0, \ldots, N.
	\end{equation*}
	Assuming a double integrator model of the form
	$\pos(t) = \pos_k + \dot{\pos}_k \cdot (t - t_k) + \int_{t_k}^{t} \int_{t_k}^\tau \ddot{\pos}(s) \; \mathrm{d}\, s \; \mathrm{d}\, \tau$
	with the shorthand $\pos_k = \pos(t_k)$
	for all $k = 0, \ldots, N$ and $t \in [t_k, \; t_{k+1}]$ it can be written as
	{\small
	\begin{equation*}
	\norm{\pos_k + \dot{\pos}_k \cdot (t - t_k) + \int_{t_k}^{t} \int_{t_k}^\tau \ddot{\pos}(s) \; \mathrm{d}\, s \; \mathrm{d}\, \tau - \cen_k}_2 \leq \dist(\cen_k) - \dmin. \quad
	\end{equation*}
	}
	Using the triangle inequality we get
	{\small
	\begin{equation*}
	\norm{\pos_k - \cen_k}_2 \leq \dist(\cen_k) - \dmin - \norm{\dot{\pos}_k}_2 \cdot (t - t_k) -  \norm{\int_{t_k}^{t} \int_{t_k}^\tau \ddot{\pos}(s) \; \mathrm{d}\, s \; \mathrm{d}\, \tau}_2.
	\end{equation*}
	}
	With $\norm{\int \ddot{\pos}(\tau) \mathrm{d}\, \tau}_2 \leq \int \norm{\ddot{\pos}(\tau)}_2 \mathrm{d}\, \tau$ and assuming that system's total acceleration is bounded $\norm{\ddot{\pos}(t)}_2 \leq \ub{a} \;\; \forall\, t \in \R$, which is a reasonable assumption for most physical systems, yields
	\begin{equation*}
	\norm{\pos_k - \cen_k}_2 \leq \dist(\cen_k) - \dmin - \norm{\dot{\pos}_k}_2 \cdot (t - t_k) -  \frac{\ub{a}}{2} \cdot (t - t_k)^2.
	\end{equation*}
	We assume that velocities are bounded in all discretization points, i.e., $\norm{\dot{\pos}_k}_2 \leq \ub{v}$ for $k=0,\ldots,N$.
	Considering that \linebreak$t_0 - \frac{\Delta t}{2}$ and $t_N + \frac{\Delta t}{2}$ lie outside of the prediction horizon and that $\norm{\dot{\pos}_{k\pm1}}_2 \leq \ub{v}$ for $k=1,\ldots,N-1$, it is sufficient to consider the interval $t_k \pm \frac{\Delta t}{2}$ in each time step $t_k$ and get
	\begin{equation*}
		\norm{\pos_k - \cen_k}_2 \leq \dist(\cen_k) - \dmin - \ub{v} \cdot \frac{\Delta t}{2}  -  \ub{a} \cdot \frac{\Delta t^2}{8} \quad \text{for } k=0,\ldots,N.
	\end{equation*}%
	\begin{figure}
		\centering
		\def\svgwidth{0.9\columnwidth}
	\executeiffilenewer{figures/ctcs_sampling.svg}{figures/ctcs_sampling.pdf}%
	{inkscape -z -D --file=/Users/tobi/Documents/pubs/schoels-iros-2019/figures/ctcs_sampling.svg %
		--export-pdf=/Users/tobi/Documents/pubs/schoels-iros-2019/figures/ctcs_sampling.pdf --export-latex}%
	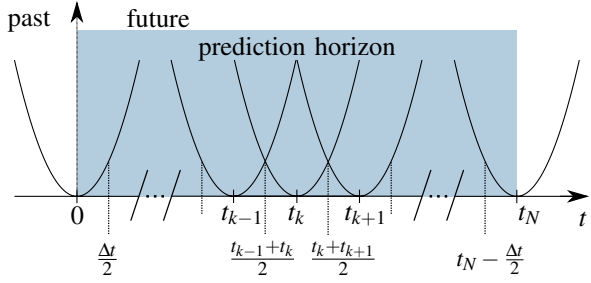%

		\caption{Maximal action radius between sampling points.}
		\label{fig:ctcs}
	\end{figure}%
	To illustrate the reasoning Fig.~\ref{fig:ctcs} sketches the maximum deviation from the current location over time.
  For the sake of simplicity we assume
	\begin{equation} \label{eq:ctime_collision_freedom}
	\dmin_k \geq \ub{v} \cdot \frac{\Delta t}{2} + \ub{a} \cdot \frac{\Delta t^2}{8} + \dmin \quad k=0,1,\ldots, N
	\end{equation}
	from now on, such that \eqref{eq:nlp_cfr_constraint} guarantees continuous-time collision freedom is under the assumptions made above.

	Note that \eqref{eq:ctime_collision_freedom} implies that $\admissible_{\cen_\mathrm{k}}$ and $\admissible_{\cen_{\mathrm{k}+1}}$ overlap or touch in the point where the trajectory transits from one into the next.
  This is achieved by taking the robot's action radius into account, see Fig.~\ref{fig:ctcs}.

	\subsection{Rotational Invariance Trick}
	Modeling a robot's kinematics and dynamics using Euler-angles and continuous variables for the orientation is intuitive, but leads to ambiguities and possibly singularities, e.g. the gimbal lock.
	Unit quaternions are widespread approach to circumvent these problems for 3-D rotations.
	For the 2-D case, however, they pose an avoidable overhead. In this case we use Euler-angles and continuous orientation variables for computational efficiency and the `\textit{rotational invariance trick}` to compensate for the ambiguity. We measure the distance between of two orientations $\theta_1, \theta_2 \in \R$ by
	\begin{equation}\label{eq:orientation_term}
	d(\theta_1, \theta_2) = \norm{\matr{\cos \theta_1 - \cos \theta_2 \\ \sin \theta_1 - \sin \theta_2}}_2.
	\end{equation}
	Note that $d(\theta_1, \theta_2) = d(\theta_1 + l_1\cdot 2\pi, \theta_2 + l_2 \cdot 2\pi), \forall l_1, l_2 \in \mathbb{Z}$.
	We thereby avoid unnecessary $360^\circ$ rotations of the robot, which is relevant for example if the robot drives in a circle.

	\subsection{Choosing the Objective Function}
	We use a quadratic cost function for reference tracking with regularization:
	{\small
	\begin{align*}
		J(\mathbf{w}) = &\sum_{k=0}^{N-1} \alpha^k \norm{\mathbf{q}(\state_k) - \mathbf{q}(\refstate_k)}_Q^2 + \norm{\controls_k - \refcontrols_k}_R^2 \\
		&+ \alpha^N \norm{\mathbf{q}(\state_N) - \mathbf{q}(\refstate_N)}_{Q_N}^2,
	\end{align*}
	}%
	where $\alpha > 1$ leads to exponentially increasing stage cost and reduces oscillating behavior around the goal, $\refstate_k, \refcontrols_k$ denote reference state and controls at stage $k$, $\mathbf{q}: \R^\nstate \to \R^{n_{\mathrm{\mathbf{q}}}}$ augments the state $\state$ by applying proper transformations 
	where applicable, $Q, Q_N\in\R^{n_{\mathrm{\mathbf{q}}} \times n_{\mathrm{\mathbf{q}}}}$, and $R\in\R^{\ncontrols\times \ncontrols}$ are positive definite matrices.

	\section{\ac{sciam}-BASED MOTION PLANNING} \label{sec:motion_planning}
	In this section we propose and detail two algorithms, one for pure trajectory optimization (Alg.~\ref{alg:sciam_trajopt}) and the other for simultaneous trajectory optimization and tracking (Alg.~\ref{alg:sciam}).

	\subsection{\ac{sciam} for Trajectory Optimization}
 The proposed trajectory optimization algorithm (see Alg.~\ref{alg:sciam_trajopt}) starts by computing a feasible initial guess and a reference trajectory (Lines 1--2).
	\begin{algorithm}
		\small
		\begin{algorithmic}[1]
			\Require $\state_\mathrm{S},\; \state_\mathrm{G},\; \Delta t, \; \varepsilon$ \Comment{start and goal state}
			\State $\vec{w}^* \gets$ \textsc{initialGuess}($\state_\mathrm{S},\; \state_\mathrm{G}, \; \Delta t$) \Comment feasible initialization
			\State $\mathbf{r} \gets$ \textsc{referenceTrajectory}($\state_\mathrm{S},\; \state_\mathrm{G}, \; \Delta t$)
			\Do
			\State $\mathbf{w} \gets \mathbf{w}^*$ \Comment{set last solution as initial guess}
			\State $\vec{w}^* \gets \textsc{\ac{sciam}-iteration}(\vec{w}\, ; \; \vec{r}, \; \inistate, \; \Delta t)$ \Comment $\inistate = \state_\mathrm{S}$
			\DoWhile {$\textsc{cost}(\mathbf{w}^*) - \textsc{cost}(\vec{w}) > \varepsilon$}
			\State \Return $\mathbf{w}^*$
		\end{algorithmic}
		\caption{\ac{sciam} for offline trajectory optimization}
		\label{alg:sciam_trajopt}
	\end{algorithm}
	In the general case of nonconvex scenarios, such as cluttered environments, feasible initializations can be obtained through a sampling-based motion planner \cite{palmieri2015distance,palmieri2017kinodynamic}.
	To monitor the progress the initial guess is copied (Line~4), before using
	it as initial guess for the \textsc{\ac{sciam}-iteration} (Line~5).
	Lines~4--5 are repeated as long as the \textsc{cost}-function shows an improvement that exceeds a given threshold $\varepsilon$ (Line~6).
	Finally the best known solution $\vec{w}^*$ is returned (Line~7).
		For trajectory optimization the terminal constraint in \eqref{eq:nlp_cfr} becomes an equality constraint, which enforces that the goal state $\state_\mathrm{G}$ is reached at the end of the horizon, i.e. $\mathbb{X}_\mathrm{T} = \{\state_\mathrm{G}\}$.

	\subsection{\ac{sciam}-NMPC}
	While  Alg.~\ref{alg:sciam_trajopt} iteratively improves a trajectory that connects $\state_\mathrm{S}$ and $\state_\mathrm{G}$, Alg.~\ref{alg:sciam}  uses a shorter, receding horizon.
	\priority[2]{This can be considered the \ac{rti} version of Alg.~\ref{alg:sciam_trajopt}.}
	Therefore the trajectory computed by \textsc{initialGuess} (Line~1) is not required to reach the goal state $\state_\mathrm{G} \in \mathbb{X}_\mathrm{G}$. \priority[2]{In many cases it is sufficient to choose $\vec{w} = [\state_0\T, \controls_\mathrm{s}\T, \ldots, \controls_\mathrm{s}\T, \state_0\T]\T$, where $\controls_\mathrm{s}$ is chosen, such that the robot remains in the current state $\state_0$.}
	%
	\begin{algorithm}
		\small
		\begin{algorithmic}[1]
			\Require $\inistate,\; \state_\mathrm{G}, \; \Delta t, \; \mathbb{X}_\mathrm{G}$ \Comment current and goal state
			\State $\vec{w} \gets$ \textsc{initialGuess}($\inistate,\; \state_\mathrm{G}, \; \Delta t$) \Comment feasible initialization
			\While {$\inistate \notin \mathbb{X}_\mathrm{G}$ }
			\State $\inistate \gets \textsc{getCurrentState}()$
			\State $\vec{r} \gets$ \textsc{referenceTrajectory}($\inistate, \; \state_\mathrm{G}, \; \Delta t$)
			\State $\vec{w}^* \gets$ \textsc{\ac{sciam}-iteration}($\vec{w}; \; \vec{r}, \; \inistate, \Delta t$) \Comment Alg.~\ref{alg:solve_sciam}
			\State $\textsc{applyFirstControl}(\vec{w}^*)$ \Comment recall $\controls_0 \in \vec{w}^*$
			\State $\vec{w} \gets \textsc{shiftTrajectory}(\vec{w}^*)$ \Comment recede horizon
			\EndWhile
		\end{algorithmic}
		\caption{\ac{sciam}-\ac{nmpc}}
		\label{alg:sciam}
	\end{algorithm}

	While the robot has not reached the goal region $\mathbb{X}_\mathrm{G}$ (Line~2), it is iteratively steered to it (Lines~3--8).
	Each iteration starts by updating the robot's current state $\state_0$.
	Based on the complexity of the scenario \textsc{referenceTrajectory} may return a guiding trajectory to the goal or just the goal state itself (Line~4).
	We run Alg.~\ref{alg:solve_sciam} to compute a new trajectory (Line~5), before sending the first control to the robot (Line~6).
	\textsc{shiftTrajectory} moves the horizon one step forward (Line~7).

	To ensure recursive feasibility, which implies collision avoidance, the terminal constraint \eqref{eq:nlp_cfr_terminal_constraint} is commonly chosen such that the robot comes to a full stop at the end of the horizon, i.e. $\mathbb{X}_\mathrm{T} = \{\state \in \R^\nstate : S_\mathrm{v}\cdot\state = 0\}$, where $S_\mathrm{v}\in \R^{n_\mathrm{v}\times\nstate}$ is the matrix that selects the velocities from the state vector.


	\section{EXPERIMENTS AND DISCUSSION} \label{sec:experiments}
	\begin{figure*}[t!]
		\centering
		\begin{subfigure}[t]{\linewidth}
			\centering
			\includegraphics[width=0.19\linewidth]{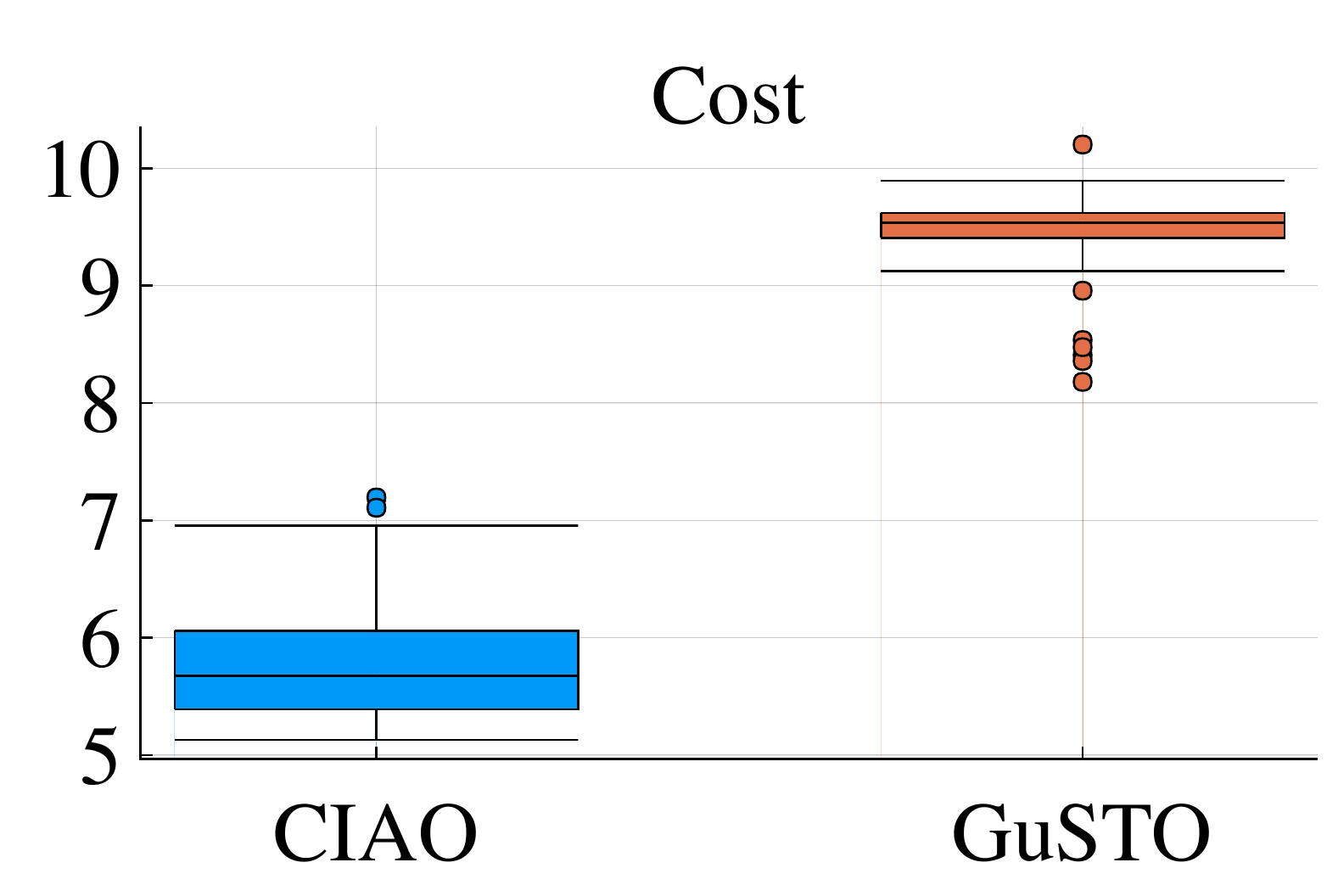}
			\includegraphics[width=0.19\linewidth]{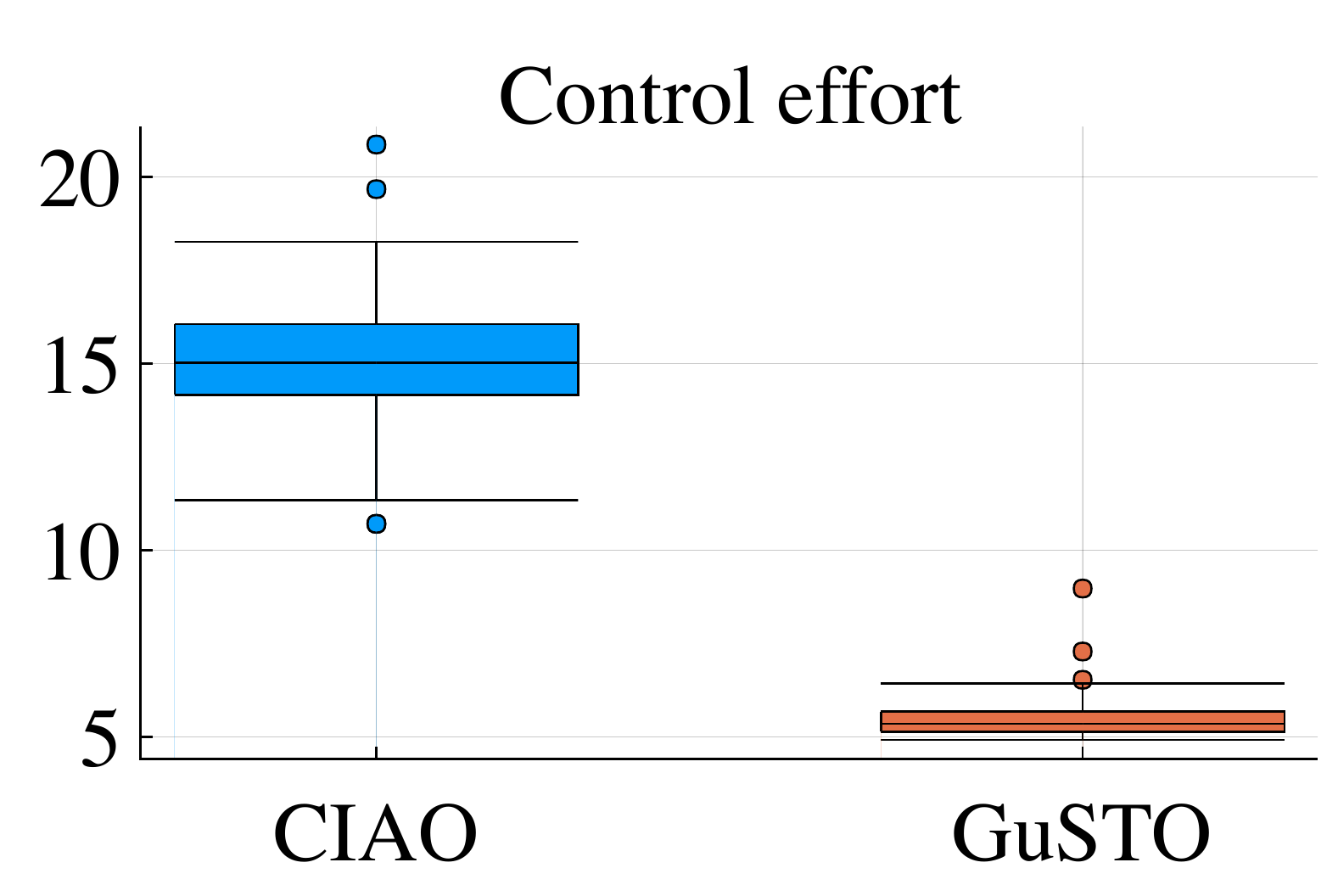}
			\includegraphics[width=0.19\linewidth]{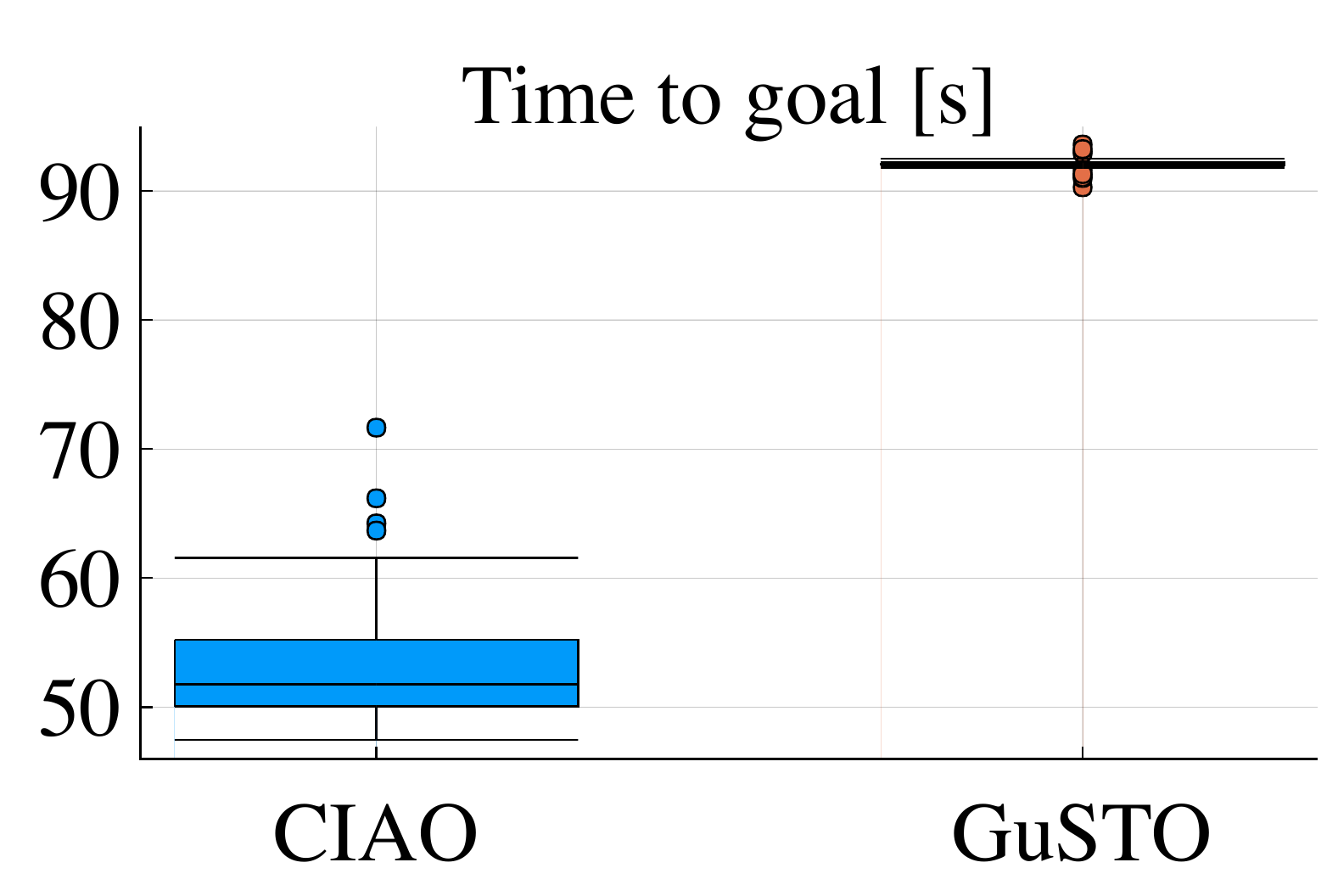}
			\includegraphics[width=0.19\linewidth]{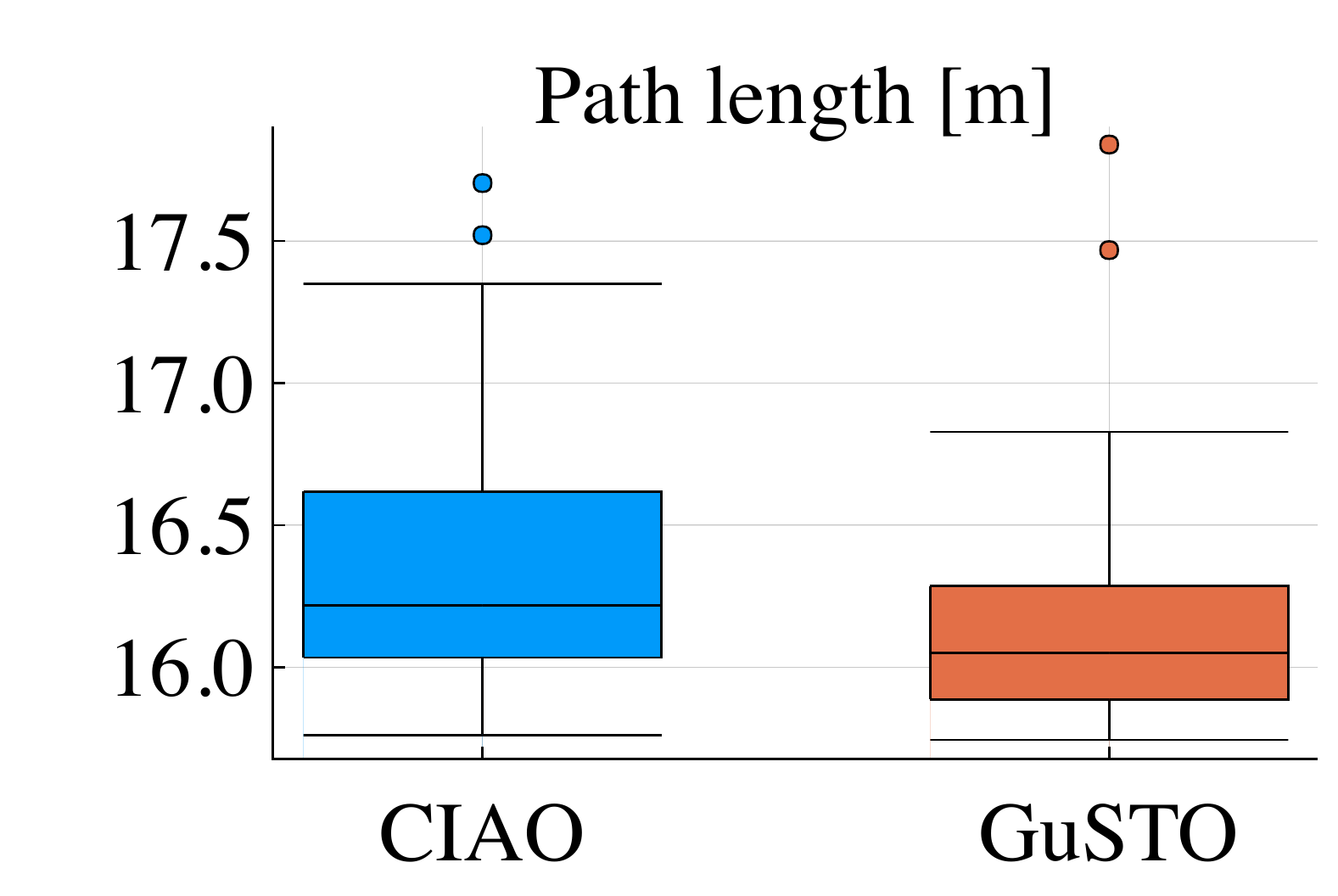}
			\includegraphics[width=0.19\linewidth]{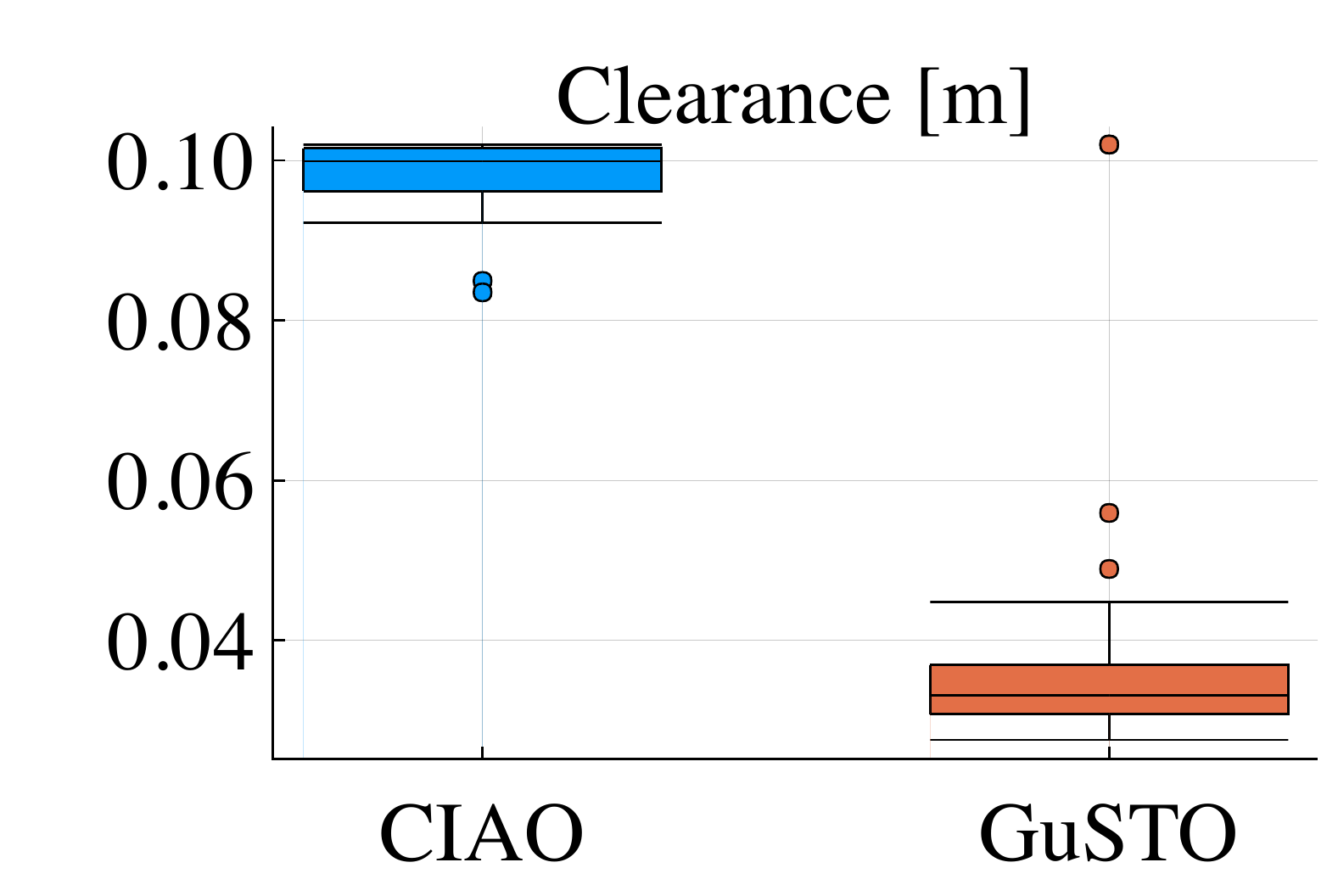}
		\end{subfigure}%
		\caption{Trajectory Optimization Benchmark Results. CIAO finds faster trajectories with higher clearance than GuSTO.\\[-0.1cm]
		}
		\label{fig:benchmark}
	\end{figure*}%
	To evaluate \ac{sciam} in terms of planning efficiency and final trajectory quality, we compare it against a set of baselines.
	We challenge \ac{sciam} by using nonlinear dynamics and a nonconvex cost function. Further we use a sampling based motion planner to initialize it with a collision free path that does not satisfy the robot's dynamics.
	In this case a \ac{nlp}-solver is required to solve the \ac{sciam}-\ac{nlp} \eqref{eq:nlp_cfr}. We use the primal-dual interior point solver Ipopt \cite{Waechter2009} with the linear solver MA-27 \cite{HSL} called through CasADi \cite{Andersson2018} (version~3.4.5).

	\noindent For the evaluation, we consider three types of experiments:
  \begin{itemize}
    \item[(A)] numerical experiments to
	investigate the behavior of the \acl{cfr} constraint against competing
	formulations; \par
  \item[(B)] a trajectory optimization benchmark to evaluate the quality of trajectories found by CIAO; \par
  \item[(C)] real-world experiments where \ac{sciam} is qualitatively compared to a
	state of the art baseline.
  \end{itemize}
%

	\subsection{Comparison of constraint formulations} \label{sec:numerical_experiments}
	In a first set of experiments, the numerical performance of the \acl{cfr} constraint formulation, as derived in Sec.~\ref{sec:cfr}, is compared to common alternatives: the actual constraint as defined in Eq.~\eqref{eq:min_dist} (actual), a linearization of the actual constraint (linear), and a log-barrier formulation (log-barrier).
	\priority{They differ only in the way the obstacle avoidance constraint \eqref{eq:nlp_cac} is formulated.}
	Our findings are reported in Tab.~\ref{tab:constraint_comparison}.

	%
\begin{table}
	\centering
	\begin{tabular}{lrrrr}
		\toprule
		&  \textit{actual} & \textit{linear} & \textit{\ac{sciam}} & \textit{log-barrier} \\
		\midrule
		ms / iteration     &             2.00 &                   0.72 &     \textbf{0.70} &          2.23 \\
		ms / step     &            40.78 &         \textbf{13.13} &             17.26 &         50.34 \\
		iterations / step  &            20.35 &         \textbf{18.23} &             24.64 &         22.57 \\
		time to goal [s]      &   \textbf{14.35} &                  14.39 &             18.67 &       (13.46)$^*$ \\
		path length [m]     &   \textbf{10.22} &                  10.33 &             10.58 &        (9.25)$^*$ \\
		max ms / step &           448.94 &                 230.07 &   \textbf{179.26} &        $>1000$   \\
		\% timeouts   &       \textbf{0} &             \textbf{0} &        \textbf{0} &             11.3 \\
		\bottomrule
	\end{tabular}
	\caption[]{Comparison of constraint formulations.\footnote{These experiments were conducted in simulation considering a robot with differential drive dynamics (5 states, 2 controls) and a prediction horizon of 50 steps, resulting in a total of 405 optimization variables.\\
			$^*$ in Table~\ref{tab:constraint_comparison}: not representative because complex scenarios with long transitions failed.}
	}
	\label{tab:constraint_comparison}
\end{table}

	The average computation time taken per \ac{mpc}-step and per Ipopt iteration are given as `ms / step' and `ms / iteration' respectively, `iters / step' are the average Ipopt iterations per \ac{mpc}-step.
	The path quality is evaluated in terms of `time to goal` and `path length`.
	Averages in the first five rows are taken over $62$ scenarios, for which the maximum CPU time of $1.0 \rm{s}$ was not exceeded. The percentage of runs that exceed the CPU time is given by `\% timeouts'.
	The maximum CPU time taken for a single \ac{mpc}-step is given by `max ms / step'.

	We observe that the actual constraint is producing both fastest and shortest paths. This path quality comes at comparatively high computational cost.
	Linearizing the actual constraint reduces the computational effort, while maintaining a high path quality. In contrast to \ac{sciam}, linearization is not an inner approximation and can lead to constraint violations that necessitate computationally expensive recovery iterations. This increases the overall computation time significantly and leads to a higher maximal computation time.
	At the cost of lower path quality, but a similar average computation time,
	\ac{sciam} overcomes this problem by preserving feasibility. This leads to
	a lower variance in the computation time, and allows for continuous time collision avoidance
	guarantees. A further advantage, which is relevant in practice, is that
	\ac{sciam} generalizes to not continuously differentiable distance function
	implementations, e.g. distance fields.

	Including the collision avoidance constraint as a barrier term in the objective, i.e. by adding $- \log(\dist(p_k) - \dmin)$ to the stage cost $l_k$, is an alternative approach to enforce collision freedom. Our results suggest, however, that for our application it is least favorable among the considered options.

	\subsection{Trajectory Optimization Benchmark}
	In a second set of experiments \ac{sciam} is compared to GuSTO \cite{Bonalli2019} using the implementation publicly provided by the authors. 
	In these experiments we consider a free-flying Astrobee Robot with 12 states and 6 controls, that has to be rotated and traversed from a start position on the bottom front left corner of a $10 \times 10 \times 10~\mathrm{m}$ cube to a goal in the opposite corner. The room between start and goal point is cluttered with $25$ randomly placed static obstacles of varying sizes (between $1$ and $2$ meters). Fig.~\ref{fig:benchmarkarena} shows some examples.

	The results reported in Tab.~\ref{tab:benchmark_results} and Fig.~\ref{fig:benchmark} were performed in Julia\priority[0]{\ \cite{Bezanson2017}} on a MacBook Pro with an Intel Core i7-8559U  clocked at $2.7 \rm{GHz}$. The \acp{scp} formulated by GuSTO \cite{Bonalli2019} are solved with Gurobi \cite{Gurobi2018}. Both algorithms are provided with the same initial guess,  which is computed based on a path found with RRT \cite{LaValle2006}.
	We used a horizon of $100~\mathrm{s}$ equally split into $250$ steps, resulting in a sampling time of $0.4~\mathrm{s}$.
	\begin{table}
		\centering
		\begin{tabular}{lcc}
\toprule
measure & \multicolumn{1}{c}{CIAO} & \multicolumn{1}{c}{GuSTO} \\
\midrule
Compute{\footnotemark} [s] & 14.792$\pm$11.966 & 131.367$\pm$130.743 \\
Iterations & 30.660$\pm$15.886 & 4.520$\pm$1.282 \\
Compute / Iteration [s] & 0.475$\pm$0.230 & 27.729$\pm$22.096 \\
Linearization Error & 4.66e-14$\pm$3.67e-15 & 4.10e-06$\pm$1.28e-06 \\
\bottomrule
\end{tabular}

		\caption{Numerical Performance: Average $\pm$ std values.}
		\footnotetext{These timings are only indicative due to differences in implementation, a similar trend is confirmed in Fig.~\ref{fig:cfr_compute}.}
		\label{tab:benchmark_results}
	\end{table}%
	Since both GuSTO and \ac{sciam} use tailored cost functions we evaluate the computed trajectories using a common cost function $J_\rho$, which is based on the state distance metric $\rho: R^\nstate \times R^\nstate \to \R$ proposed by \cite{LaValle2001}:
	$J_\rho(\vec{w}; \state_\mathrm{G}) = \sum_{k=0}^{N} \rho(\state_k, \state_\mathrm{G})$, with goal state $\state_\mathrm{G}$ and all weights of the distance metric chosen equal.
	The controls are evaluated separately and reported as control effort given by $J_\mathrm{\controls}(\vec{w}) = \sum_{k=0}^{N-1} \Delta t \cdot \norm{\controls_k}_1$.
	The path quality is evaluated in terms of time to goal, path length, and clearance (minimum distance to the closest obstacle along the trajectory). The first two measures take the time and path length until the state distance metric falls below a threshold of $0.5$, while the latter is evaluated on the entire trajectory. These three metrics are evaluated on an oversampled trajectory using a sampling time $\Delta t =0.01\,\mathrm{s} $.

The results in Fig.~\ref{fig:benchmark} show that \ac{sciam} finds faster trajectories than GuSTO and thereby also achieves significantly lower cost.
As depicted in Fig.~\ref{fig:benchmarkarena} it maintains a larger distance to obstacles for higher speeds.
	This behavior allows for a higher average speed, at the cost of a higher control activation and slightly longer paths in comparison to GuSTO.

	In our experiments both \ac{sciam} and GuSTO find solutions to all considered scenarios.
	As reported in Tab.~\ref{tab:benchmark_results}, \ac{sciam} (Alg.~\ref{alg:sciam_trajopt}) requires more iterations
	to converge, but the individual iterations are cheaper.
	Moreover \ac{sciam} obtains a feasible trajectory after
	the first iteration and therefore could be terminated early, while GuSTO does
	not have this property and takes several iterations to find a feasible
	trajectory.
	Even though the dynamics are mostly linear we observe linearization errors for GuSTO, originating from the linear model they use.

	In summary CIAO finds trajectories of higher quality than GuSTO at lower computational effort.

%
%
%
	\subsection{Real-World Experiments - Differential Drive Robot}
	\begin{figure}[t!]
		\centering
		\begin{subfigure}[t]{\columnwidth}
			\centering
			\includegraphics[trim=13 0 0 0 , clip,width=0.32\columnwidth]{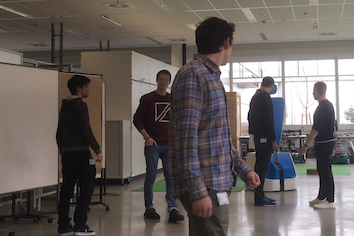}
			\includegraphics[trim=13 0 0 0 , clip,width=0.32\columnwidth]{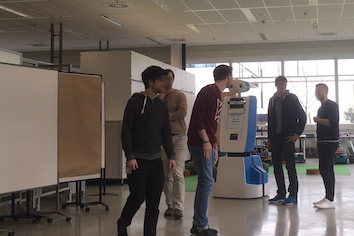}
			\includegraphics[trim=13 0 0 0 , clip,width=0.32\columnwidth]{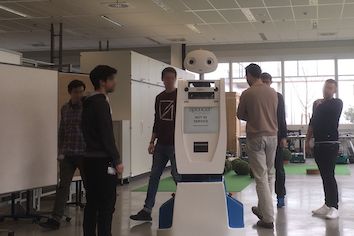}
		\end{subfigure}%
		\vspace{5pt}
		\begin{subfigure}[t]{\columnwidth}
			\centering
			\includegraphics[trim=100 70 300 80, clip,width=0.32\columnwidth]{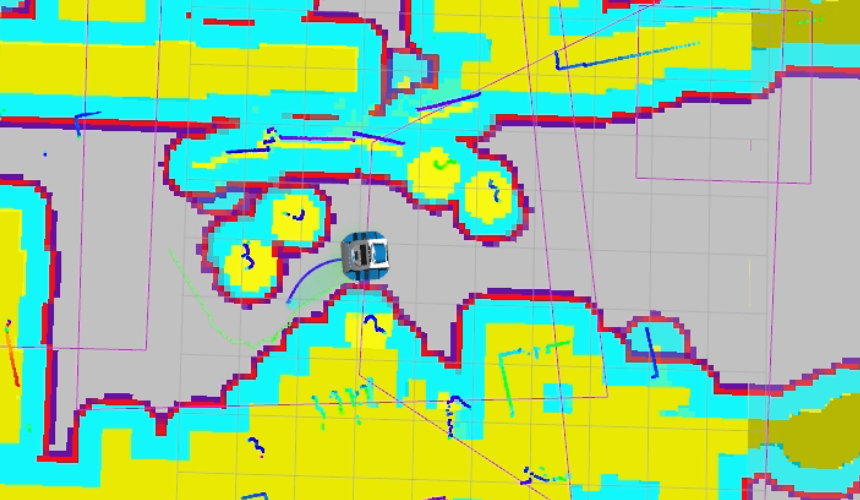}
			\includegraphics[trim=100 70 300 80, clip,width=0.32\columnwidth]{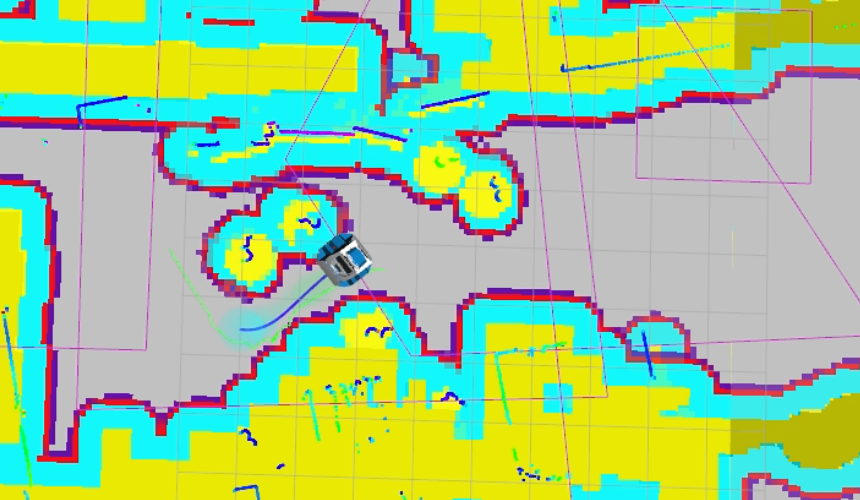}
			\includegraphics[trim=100 70 300 80, clip,width=0.32\columnwidth]{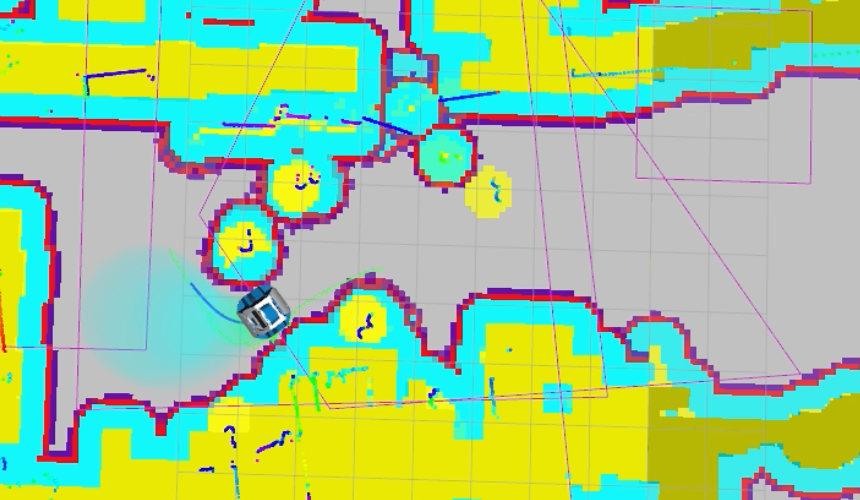}
		\end{subfigure}%
		\caption{\acs*{sciam} steers a wheeled mobile robot through a group of people. Real-world (top) and RViz (bottom): Planned trajectory as blue line, \aclp*{cfr} as transparent circles, obstacles in yellow, safety margin in light blue.}
		\label{fig:sciam_example}
	\end{figure}
	To qualitatively assess the behavior of \ac{sciam}-\ac{nmpc} (Alg.~\ref{alg:sciam}), it was tested in dynamic real-world scenarios with freely moving humans. A representative example is depicted in Fig.~\ref{fig:sciam_example}.
	Note that \ac{sciam} has no knowledge of the humans' future movements.
	It is instead considering all humans as static obstacles in their current position.

	As in Sec.~\ref{sec:numerical_experiments} a differential drive robot is used, this time with a horizon of $5\,\mathrm{s}$ and a control frequency of $10~\mathrm{Hz}$ resulting in a total of $405$ optimization variables (including slacks).
	For these experiments \ac{sciam} was implemented as a C++ ROS-module, the distance function was realized as distance field based on the code by \cite{Lau2013}. Initial guesses and reference paths were computed using an A* algorithm \cite{Hart1968}.

	Since GuSTO is not suitable for \acf{rhc}, we used an extended version of the elastic-band (EB) method \cite{Quinlan1993}. To obtain comparable results, we used the same A* planner and localization method with both algorithms.

	In Fig.~\ref{fig:sciam_example}, it can be seen that the \acfp{cfr} (transparent circles) keep to the center of the canyon-like free space.
	The predicted trajectory (blue line) is deformed to stay inside the \acp{cfr}. This is a predictive adaptation to the changed environment. For the shown, representative example in Fig.~\ref{fig:sciam_example} the robot passed smoothly the group.
	Comparable scenarios were solved similarly by EB.

	We observed that groups of people pose a particular challenge that could, however, be solved by both approaches.
	We note that the robot is moving a bit faster in proximity to people for EB, while
	CIAO adjusts to blocked paths a bit faster.
	The most significant difference between the methods is that \ac{sciam} combines rotation and backward/forward motion, while the EB rotates the robot on the spot.
	Both methods succeeded in steering the robot through the group safely, without a single collision.

	\begin{figure}
		\centering
			\vspace{-1.5em}
		\scalebox{0.5}{\input{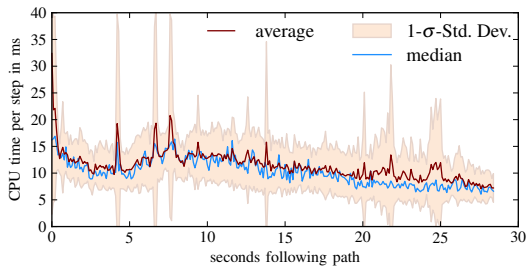}}
		\caption{\ac{nlp}-solver times obtained in a simulated environment including non-deterministically moving humans.
			A comparison of these computation times with the ones reported in \cite{Bonalli2019} indicates, that \ac{sciam} is computationally cheaper than GuSTO.
		}
		\label{fig:cfr_compute}
	\end{figure}
	Fig.~\ref{fig:cfr_compute} shows representative computation times obtained in simulation on a set of 14 scenarios involving non-deterministically moving virtual humans.
	The reported computation time accounts only for solving the \ac{sciam}-\ac{nlp} and function evaluations in CasADi, the processing time required by preprocessing steps and other components is not included.
	High computation times originate from far-from-optimal initializations occurring when a new goal is set, i.e. around time $t=0$, or if humans cross the planned path.
	\priority[0]{The computation time per \ac{nlp}-solver iteration is almost constant and approximately $1.1~\rm{ms}$ per iteration.}

In summary CIAO and the elastic band (EB) approach show similar behavior.
In contrast to EB, CIAO computes kinodynamically feasible and guaranteed continuous time
collision free trajectories. Further it has a notion of time for the planned
motion, such that predictions for dynamic environments can be incorporated in future work.

	\section{CONCLUSIONS AND FUTURE WORK} \label{sec:conclusion}

	This work proposes \ac{sciam}, a new framework for trajectory optimization,
	that is based on a novel constraint formulation, that allows for NMPC based collision avoidance in real-time.
	We show that it reaches or exceeds state of the art performance in
	trajectory optimization at significantly lower computational effort, scales
	to high dimensional systems, and that it can be used for \ac{rhc} style
	\ac{mpc} of mobile robots in dynamic environments.

	Future research will focus on extending \ac{sciam} to full body collision checking, guaranteed obstacle avoidance in dynamic environments, time optimal motion planning, and multi body robots.
	A second focus will lie on efficient numerical methods that exploit the structure and properties of the \ac{sciam}-NLP stated above.


  \balance
	\bibliographystyle{IEEEtran}
	\bibliography{syscop,localBib}


	\acrodef{cfr}[FB]{free ball}
	\acrodef{dp}[DP]{dynamic programming}
	\acrodef{ilqr}[iLQR]{iterative linear quadratic regulator}
	\acrodef{ip}[IP]{interior point}
	\acrodef{licq}[LICQ]{linear independence constraint qualification}
	\acrodef{mpc}[MPC]{model predictive control}
	\acrodef{nmpc}[NMPC]{nonlinear \ac{mpc}}
	\acrodef{nlp}[NLP]{nonlinear program}
	\acrodef{ocp}[OCP]{optimal control problem}
	\acrodef{qcqp}[QCQP]{quadratically constrained quadratic program}
	\acrodef{rhc}[RHC]{receding horizon control}
	\acrodef{rk4}[RK4]{Runge-Kutta method of 4\textsuperscript{th} order}
	\acrodef{ros}[ROS]{Robot Operating System}
	\acrodef{rti}[RTI]{real-time iteration}
	\acrodef{sam}[SAM]{sequential approximate method}
	\acrodef{sciam}[CIAO]{Convex Inner ApprOximation}
	\acrodef{scp}[SCP]{sequential convex programming}
	\acrodef{sqp}[SQP]{sequential quadratic programming}
	\acrodef{socp}[SOCP]{second order cone programming}

\end{document}